\documentclass[11pt]{phai}

\usepackage{amsthm}
\usepackage{enumitem}
\usepackage{float}
\usepackage{placeins}
\usepackage{xurl}
\usepackage{xspace}

\setlist[itemize]{leftmargin=1.4em,itemsep=0.25em,topsep=0.25em}
\newcommand{\method}{\textsc{JEPA}-Anything\xspace}
\newcommand{\R}{\mathbb{R}}
\newtheorem{proposition}{Proposition}
\newtheorem{corollary}{Corollary}

\title{JEPA-Anything: Learning Predictive Models across Different Worlds}

\authors{%
{\fontsize{10.2}{13}\selectfont
\mbox{Taoyong Cui\aff{1,2}}\quad
\mbox{Zhongyao Wang\aff{3}}\quad
\mbox{Xinyue Xu\aff{2}}\quad
\mbox{Weiyang Liu\aff{2}}\quad
\mbox{Zhaochen Yu\aff{1}}\quad
\mbox{Yuying Zhang\aff{4}}\quad
\mbox{Qiang Gao\aff{3}}\quad
\mbox{Mengyue Yang\aff{5}}\quad
\mbox{Wanli Ouyang\aff{2}}\quad
\mbox{Pheng Ann Heng\aff{2,*}}\\
\mbox{Yingcheng Wu\aff{1,6*}}\quad
\mbox{Zhenfei Yin\aff{1,7*}}\quad
\mbox{Ling Yang\aff{1,8*}}%
\par}
\vspace{0.8em}
{\normalsize
\setlength{\fboxsep}{6pt}%
\colorbox{blue!5}{%
  \hspace{8pt}%
  \href{https://github.com/Gen-Verse/JEPA-Anything}{%
    \textcolor{blue!65!black}{\textbf{Code: JEPA-Anything}}}%
  \hspace{14pt}\textcolor{black!30}{|}\hspace{14pt}%
  \href{https://huggingface.co/collections/Gen-Verse/jepa-anything}{%
    \textcolor{blue!65!black}{\textbf{Models: Jepa-Anything}}}%
  \hspace{8pt}%
}%
\par}
}

\legend{* Corresponding authors. Organizations and contact details are listed
in Appendix~\ref{app:author-information}.}

\runningtitle{JEPA-Anything}

\checkdata[Main Contact]{wuyc@phai-labs.com;
yin@phai-labs.com; yang@phai-labs.com}

\begin{document}
\maketitle

\begin{abstract}
World modeling enables intelligence to anticipate consequences, guide interventions, and learn from interaction. Yet predictive models remain domain-specific: \emph{can a common learning principle support world modeling across radically different systems?} We introduce \textbf{\method}, a domain-agnostic framework based on \emph{orthogonal predictive factorization} (OPF). Extending joint-embedding predictive architectures, OPF decomposes latent targets into complementary factors, learns them through dedicated pathways, and recombines them within a shared predictive design. We evaluate \method across \textbf{seven domains}: vision, biology, clinical trajectories, control, molecular dynamics, physical fields, and weather. Experiments span representation learning, intervention prediction, out-of-distribution generalization, and long-horizon dynamics, including \textbf{10 matched dynamics tasks}, forecasting of \textbf{over 1,000 clinical events}, and \textbf{100-step molecular rollouts across four systems}. Against matched JEPA baselines, \method improves reported metrics on all 10 dynamics tasks and reduces single-intervention prediction error on Interventional Pong by \textbf{34.8\%}. It achieves the lowest one-step and 100-step molecular errors among compared methods in all four systems. Beyond prediction, a factor-nominated biological intervention receives experimental support in cell co-cultures, patient-derived organoids, tumor fragments, and mice; latent orbital modes recover the Keplerian scaling exponent with a fitted slope of \textbf{$-1.4991$}. These results support a common factorized predictive principle across heterogeneous worlds, connecting world modeling with intervention and experimentally grounded scientific discovery.

\end{abstract}


\begin{figure}[p]
    \centering
    \includegraphics[
        width=1.0\linewidth,
        height=0.87\textheight,
        keepaspectratio
    ]{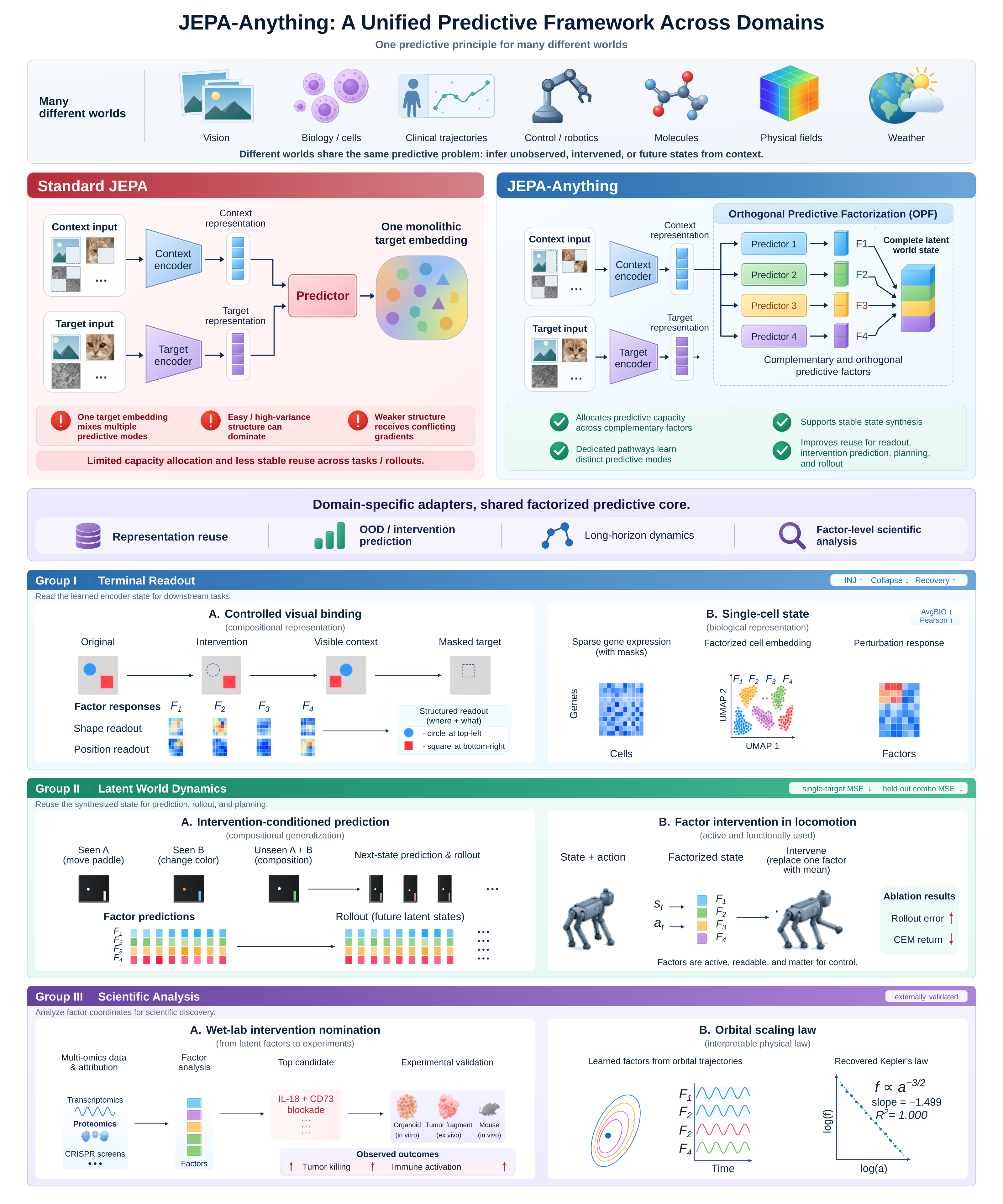}
    \caption{
        \textbf{Overview of \method.}
        \textbf{Upper section:}
        Diverse domains share the same predictive problem: inferring
        unobserved, intervened, or future states from context.
        Standard JEPA mixes multiple predictive modes in a monolithic
        target embedding. In contrast, \method uses orthogonal predictive
        factorization (OPF) to allocate predictive capacity across
        complementary, orthogonal factors, each with a dedicated predictor,
        and synthesize a complete latent world state.
        The shared factorized core supports representation reuse,
        out-of-distribution and intervention prediction, long-horizon
        dynamics, and factor-level scientific analysis.
        \textbf{Lower section:}
        Three evaluation groups illustrate these capabilities.
        Group~I (terminal readout) examines controlled visual binding
        and single-cell state representations.
        Group~II (latent world dynamics) evaluates intervention-conditioned
        prediction and factor intervention in locomotion.
        Group~III (scientific analysis) investigates wet-lab intervention
        nomination and the recovery of an orbital scaling law.
        Within each group, experiments A and B are arranged from left
        to right.
    }
    \label{fig:framework}
\end{figure}

\section{Introduction}

World models learn internal states that allow an agent or scientific model to
anticipate unobserved consequences of an observed situation. Early latent world
models compressed observations and learned recurrent dynamics for control
\cite{ha2018worldmodels}; more recent systems have shown that learned latent
dynamics can support planning across diverse control domains
\cite{hafner2025worldmodels}. Across these formulations, the central object is
a predictive state that retains the information needed to anticipate how the
modeled system can change.

This perspective extends beyond action-conditioned control. A predictive state
may summarize a physical configuration, a patient history, a molecular
trajectory, a partially observed scene, or a cellular profile. The requested
target may be a future state, a hidden spatial region, another view, or a more
complete observation of the same system. We therefore use \emph{latent world
model} in a broad but operational sense: a model that constructs a latent state
from context and uses it to predict another state of the same underlying world.
A meaningful context--target relation links the observed context to another
state of the same system, and the learned state supports reuse across target
queries, task-specific readouts, or repeated state transitions.

Within this shared interface, each system retains its own predictive structure.
A visual state may combine location, object, and
transformation; a biological state may combine cell identity and response; a
physical state may combine entities, scales, and dynamical modes. The number,
scale, and difficulty of these predictable components vary across domains, and
different downstream tasks may reuse different mixtures of them. Some targets
are relatively focused, such as a local masked region or the short-term effect
of a single intervention. Others, such as an unseen combination of
interventions, a future patient state supporting many event risks, or a full
weather field or molecular configuration, involve multiple entities, variables,
and scales with unequal predictability.

Joint-embedding predictive architectures (JEPAs) provide a natural mechanism
for learning such states~\cite{assran2023ijepa,bardes2024vjepa}. A context encoder summarizes what is observed, a
target encoder defines the state to be predicted, and a predictor maps between
the two in representation space. Latent prediction allows the model to focus on
shared, predictable structure while abstracting away noise, local texture, and
other raw-space details that may be irrelevant to downstream use. Reconstruction
can overemphasize high-variance directions that are weakly aligned with
perceptual usefulness \cite{balestriero2024reconstruction}, whereas the learned
target encoder determines the abstraction level of a JEPA target. This view also
connects latent prediction to energy-based learning
\cite{lecun2006energy}, learned similarity metrics
\cite{chopra2005similarity,hadsell2006drlim}, masked image prediction
\cite{assran2023ijepa}, and video representation learning
\cite{bardes2024vjepa}.

JEPA supplies the common latent-prediction mechanism, but the usual formulation
still represents the requested world state through one target embedding and one
prediction pathway. When local and global structure, multiple entities, or
changes at different scales share this monolithic target, easily predicted or
high-variance structure may dominate optimization, several latent directions
may serve similar roles, and weaker predictive structure may receive
conflicting gradients. This makes latent world-state design a predictive
capacity-allocation problem: the shared interface should remain fixed while the
state can be organized into multiple complementary components.

\begin{quote}
\emph{Can one latent world-model interface organize the differently structured
predictive states of many context--target systems through multiple factors?}
\end{quote}

We answer this question with \method, a latent world-modeling framework based on
\emph{orthogonal predictive factorization}. A collection of learned basis
matrices analyzes each target state into multiple components, and a dedicated
branch predicts each component from the shared context representation.
Within-factor and cross-factor orthogonality objectives allocate different
directions of the target space to different branches. Factor-activity
regularization maintains variation in projected targets, while an online
variance term discourages collapse of the trainable encoder. The predicted
components can be synthesized into a complete latent state for decoding,
planning, or autoregressive rollout. Factor identities emerge from
predictability, and the factors jointly shape the complete online state.
Ordinary downstream tasks read from that encoder state,
whereas operational transitions and factor-level analyses retain and combine
the explicit factor coordinates. The number and width of factor blocks can be
configured to the predictive complexity of an instantiation under the same
learning principle and latent world-state interface.

The word \emph{Anything} denotes the breadth of the latent world-state
interface. Each domain defines its observation tokens, context--target views,
structural descriptors, and encoder; once the context and target states are
available, the same factorized predictive core applies. Temporal conditioning,
masking, and partial observation instantiate future, spatial, biological, and
other structured targets. Adapters define observation geometry, OPF organizes
predictive capacity, and readout, synthesis, or factor analysis determines how
the learned state is used.

We evaluate this argument at three levels. First, visual binding and single-cell
tasks read the online encoder state, while longitudinal health reads a one-step
synthesized future state; in all three, evaluation ends at a task-specific
readout. Second,
intervention-conditioned sequences, control, molecular dynamics, physical
fields, and weather test whether the synthesized state supports compositional
prediction and repeated transition. Third, biological and orbital analyses test
whether the factor coordinates provide a reusable diagnostic interface. The
results improve readout quality across vision, cells, and disease forecasting;
reduce intervention-prediction error on both observed and unseen combinations;
improve the reported metrics across the ten-task matched dynamics benchmark and
all four molecular systems; and connect learned coordinates to wet-lab evidence
and a known physical scaling law.

Figure~\ref{fig:framework} summarizes the shared predictive architecture and
the three evaluation modes used throughout the paper.

Our contributions are summarized as follows:
\begin{itemize}[itemsep=0.1em,topsep=0.2em,parsep=0pt,partopsep=0pt]
    \item We define latent world modeling through a common context--target
    state interface: domain adapters handle observation geometry, while a
    modality-independent core organizes the resulting predictive states.
    \item We introduce orthogonal predictive factorization, which partitions a
    latent target into learned subspaces with dedicated predictors, providing
    configurable predictive capacity and a complete state for reuse.
    \item We combine factorized prediction with within- and cross-subspace
    orthogonality, factor activity, and online encoder variance, yielding a
    common interface for stable state synthesis and factor-level diagnostics.
    \item We instantiate the same core across physical-field, weather,
    molecular, biological, visual, clinical, and control systems, evaluating
    readout, OOD forecasting, planning, rollout stability, and scientific
    diagnostics.
\end{itemize}

\section{JEPA-Anything}
\label{sec:method}

\subsection{One Interface, One Predictive Core}

The unifying object in \method{} is a latent world-state interface: structured
observations specify an available context, a requested target, and any target
descriptors. The top row of Figure~\ref{fig:framework} shows this shared
architecture. Adapters expose a domain's predictive state, while OPF distributes
its capacity across complementary coordinates, allowing observation geometry
and predictive complexity to vary without changing the interface.

Let $\delta$ index a domain and let $x\sim\mathcal{D}_{\delta}$ be a raw
observation. A domain adapter $\mathcal{A}_{\delta}$ maps $x$ into content tokens
$H=\{h_i\}_{i\in\Omega}$ and structural descriptors
$S=\{s_i\}_{i\in\Omega}$ over an index set $\Omega$. A descriptor may encode a
patch coordinate, time stamp, graph position, entity identity, or may be empty
when no additional structure is needed. A view sampler
$\mathcal{V}_{\delta}$ selects context indices $C\subset\Omega$ and target
indices $T\subset\Omega$. This gives the common interface
\begin{equation}
    x\xrightarrow{\mathcal{A}_{\delta}}(H,S)
    \xrightarrow{\mathcal{V}_{\delta}}(H_C,S_C,T,S_T).
    \label{eq:interface}
\end{equation}
The semantics of $\mathcal{A}_{\delta}$ and $\mathcal{V}_{\delta}$ are
domain-specific; everything after Equation~\eqref{eq:interface} follows one
algorithm.

An online encoder produces a context representation
$z_c=f_\theta(H_C,S_C)$. A target encoder produces a latent target
$z_t=f_{\bar\theta}(H,S)_t\in\R^d$ for every $t\in T$. Using a full target view
includes the familiar JEPA case in which target tokens are selected after
encoding. The target parameters are updated as an exponential moving average,
\begin{equation}
    \bar\theta \leftarrow m\bar\theta+(1-m)\theta,
    \qquad 0\leq m<1,
\end{equation}
and receive no gradients. The single-vector formulation is recovered when
$|T|=1$.

Table~\ref{tab:invariance} states the boundary precisely. Hyperparameters such
as latent width or number of factors may be selected for computational scale;
the functional form of the additive OPF objective is unchanged.

\begin{table}[H]
\centering
\caption{What changes and what remains invariant across \method{}
instantiations.}
\label{tab:invariance}
\small
\resizebox{\textwidth}{!}{%
\begin{tabular}{@{}lll@{}}
\toprule
Component & Domain-specific boundary & Shared JEPA-Anything core \\
\midrule
Raw representation & Image, sequence, graph, set, record & --- \\
Tokenization and descriptors & $\mathcal{A}_{\delta}$ & Common token interface \\
Context--target semantics & $\mathcal{V}_{\delta}$ & Context predicts targets \\
Encoder architecture & ViT, Transformer, GNN, MLP, $\ldots$ & Online/EMA-target scheme \\
Predictive learning & --- & Factorized predictors and state synthesis \\
Regularization & --- & Orthogonality, factor activity, encoder variance \\
Training objective & Original loss $\mathcal{L}_{\mathrm{base}}^{(\delta)}$
& Additive OPF loss in Equation~\eqref{eq:objective} \\
Downstream use & Pooling, readout, probe, decoder, planner &
Online-encoder features or synthesized predicted states \\
\bottomrule
\end{tabular}
}
\end{table}

The interface requires a meaningful predictive relation: the target must have
usable statistical dependence on the context and represent another state of the
same underlying system. It covers pixels, trajectories, graphs, sets, fields,
and multivariate records; factor identities emerge from predictive learning
rather than predefined semantics.

\subsection{Orthogonal Predictive Factorization}

Standard JEPA training predicts a target embedding through a single predictor.
We instead introduce $K$ learned projectors $P_k\in\R^{d\times r}$ and choose
$r$ such that $Kr=d$. In the zero-penalty limit, the resulting mutually
orthogonal subspaces form a full partition of the target space. We factorize the stop-gradient target
representation as
\begin{equation}
    \widetilde z_t=\operatorname{sg}(z_t),
    \qquad
    z_t^{(k)} = P_k^\top \widetilde z_t,
    \qquad k=1,\ldots,K.
\end{equation}
Stop-gradient is applied to the target-encoder output, while the projectors
remain trainable and receive gradients from the predictive, orthogonality, and
factor-activity terms. The learned subspaces are therefore predictive and
orthogonality-regularized.
Each factor has a corresponding predictor $q_k$ that maps the shared context
representation and, when needed, target descriptor $s_t$ into the same
$r$-dimensional space,
\begin{equation}
    \widehat z_t^{(k)}=q_k(z_c,s_t).
\end{equation}
The factor predictions are concatenated as
$\widehat u_t=[(\widehat z_t^{(1)})^\top,\ldots,
(\widehat z_t^{(K)})^\top]^\top\in\R^d$. The complete target state is
synthesized through the Moore--Penrose pseudoinverse of the analysis map
$P^\top$,
\begin{equation}
    \widehat z_t
    =\left(P^\top\right)^\dagger\widehat u_t
    \in\R^d,
    \qquad P=[P_1,\ldots,P_K].
    \label{eq:synthesis}
\end{equation}
When $P$ is exactly orthogonal, $(P^\top)^\dagger=P$, and
Equation~\eqref{eq:synthesis} reduces to
$\widehat z_t=\sum_k P_k\widehat z_t^{(k)}$.
This state is the common interface for applications that decode a future
latent state or feed predictions back into a multi-step rollout.
The descriptor specifies the requested target location or identity. Predictors may ignore it when the context
uniquely determines the target. They may be implemented independently or as a
shared trunk followed by branch-specific heads.

For temporal, action-conditioned, or intervention-conditioned instances, let
$\xi_t$ denote the exogenous input at step $t$---for example, an action,
intervention label, or known forcing. The domain adapter places $\xi_t$ in the
context tokens or target descriptor. When the current context is summarized by
a latent state $z_t$, the operational transition takes the form
\[
    \widehat u_{t+1}
    =\big[q_1(z_t,\xi_t,s_{t+1});\ldots;
           q_K(z_t,\xi_t,s_{t+1})\big],
    \qquad
    \widehat z_{t+1}=(P^\top)^\dagger\widehat u_{t+1}.
\]
Repeated application defines a latent rollout. A planner may score the resulting
trajectory with a known domain reward or a task-specific reward readout; this
division lets OPF supply the predictive state while the application supplies the
reward specification.

To preserve both the direction and magnitude required by state synthesis, we
directly regress each predicted factor to its target,
\begin{equation}
    \mathcal{L}_{\mathrm{pred}}
    =\frac{1}{K|T|r}\sum_{t\in T}\sum_{k=1}^{K}
    \left\|\widehat z_t^{(k)}-z_t^{(k)}\right\|_2^2,
\end{equation}
where $\operatorname{sg}$ denotes stop-gradient on the EMA target-encoder
output.

Distinct predictive factors require an explicit diversity constraint because
several projectors can otherwise learn the same target directions. We require
the columns within each projector to be approximately orthonormal and different
projectors to occupy approximately orthogonal subspaces,
\begin{equation}
\begin{split}
    \mathcal{L}_{\mathrm{orth}}
    ={}&\sum_{k=1}^{K}
    \left\|P_k^\top P_k-I_r\right\|_F^2 \\
    &+\sum_{1\leq i<j\leq K}
    \left\|P_i^\top P_j\right\|_F^2.
\end{split}
\end{equation}
The first term avoids degenerate bases within a factor, while the second
discourages different factors from repeatedly encoding the same directions.
This design is related in spirit to redundancy-reduction objectives in
self-supervised learning \cite{zbontar2021barlow,bardes2022vicreg}, but applies
orthogonality directly to learned predictive target subspaces, complementing
objectives defined on the statistics of a single embedding.

For a strict orthogonal decomposition, the concatenated basis can be factorized
as $P_{\rm raw}=QR$, where $Q^\top Q=I_d$, and partitioned as
$Q=[Q_1,\ldots,Q_K]$. The resulting coordinates $Q_k^\top z$ admit transpose
synthesis $z=QQ^\top z=\sum_k Q_kQ_k^\top z$. The controlled geometry analysis
uses this strict decomposition. Predictive models use the learned
orthogonality-regularized projectors and pseudoinverse synthesis in
Equation~\eqref{eq:synthesis}.

Let $P=[P_1,\ldots,P_K]\in\R^{d\times d}$. The geometric role of the
constraint can be stated directly.

\begin{proposition}[Orthogonal predictive decomposition]
\label{prop:orthogonal-decomposition}
If $P_i^\top P_j=0$ for $i\neq j$ and $P_k^\top P_k=I_r$, then the factor
spaces form an orthogonal direct sum and, for every $z\in\R^d$,
\begin{equation}
    \|P^\top z\|_2^2
    =\sum_{k=1}^{K}\|P_k^\top z\|_2^2
    = \|z\|_2^2,
    \qquad
    z=\sum_{k=1}^{K}P_kP_k^\top z.
    \label{eq:energy}
\end{equation}
Thus the factors preserve all information in $z$ up to an orthogonal change of
basis. In this zero-penalty limit, the pseudoinverse synthesis in
Equation~\eqref{eq:synthesis} reduces to the corresponding orthogonal synthesis
rule.
\end{proposition}

\begin{proof}
The assumptions and $Kr=d$ give $P^\top P=PP^\top=I_d$, so $P$ is orthogonal.
The norm identity follows by blockwise expansion of $P^\top z$, and
$z=PP^\top z=\sum_k P_kP_k^\top z$ gives exact reconstruction.
\end{proof}

\begin{corollary}[Stable synthesis versus unconstrained heads]
Let $u=P^\top z$ and let $e\in\R^d$ denote an error in the concatenated
predicted factors. Under the assumptions of
Proposition~\ref{prop:orthogonal-decomposition},
\[
    \widehat z=P(u+e),
    \qquad
    \|\widehat z-z\|_2=\|e\|_2,
    \qquad
    \kappa_2(P)=1.
\]
Without the cross-factor orthogonality constraint, there exist admissible
projectors with repeated directions for which $P$ is rank deficient, and
full-rank projectors with nearly repeated directions for which
$\kappa_2(P)$ is arbitrarily large. Consequently, unconstrained heads provide
no uniform guarantee of either complete recovery or stable synthesis.
\end{corollary}

\begin{proof}
Proposition~\ref{prop:orthogonal-decomposition} makes $P$ orthogonal, so
$(P^\top)^\dagger=P$,
$P(u+e)=z+Pe$, and orthogonality preserves $\|e\|_2$. Without the cross-factor
constraint, duplicate directions can make $P$ rank deficient, while nearly
duplicate directions drive its smallest singular value toward zero and its
condition number without bound.
\end{proof}

Together, Proposition~\ref{prop:orthogonal-decomposition} and the corollary
establish non-overlapping coverage and stable synthesis;
Table~\ref{tab:citris_geometry} reports these properties in a trained model.
Subsequent experiments evaluate factor activity and domain-level interpretation.

\subsection{Maintaining Factor and Encoder Activity}

Orthogonality specifies how the subspaces relate to one another, while a
per-factor activity floor keeps every projected target coordinate active across
samples.
For a mini-batch
$\mathcal{B}$ and all its sampled target indices, let $\sigma^{\rm fac}_{k,j}$
be the empirical standard deviation of coordinate $j$ in $z_t^{(k)}$.
Because the target representation is stopped, this term shapes the projectors.
We also define a separate activity statistic on the online context
representations.
Let $\sigma^{\rm enc}_{j}$ denote the empirical standard deviation of coordinate
$j$ in the online context representations $z_c$; for token-valued contexts, the
statistic is computed over all valid context tokens in the mini-batch. Standard
deviations are evaluated as $\sqrt{\operatorname{Var}+\epsilon}$. We define
\begin{equation}
\begin{aligned}
    \mathcal{L}_{\mathrm{fac}}
    &=\frac{1}{Kr}\sum_{k=1}^{K}\sum_{j=1}^{r}
    \max(0,\gamma_{\rm fac}-\sigma^{\rm fac}_{k,j}),\\
    \mathcal{L}_{\mathrm{enc}}
    &=\frac{1}{d}\sum_{j=1}^{d}
    \max(0,\gamma_{\rm enc}-\sigma^{\rm enc}_{j}).
\end{aligned}
\end{equation}
The first term keeps projected targets active; the second sends a direct
anti-collapse gradient to the online encoder. We collect the factorized
prediction and regularization terms into an additive OPF loss. If
$\mathcal{L}_{\mathrm{base}}^{(\delta)}$ denotes the loss already used by the
original implementation in domain $\delta$, the optimized training loss is
\begin{equation}
\begin{aligned}
    \mathcal{L}_{\mathrm{OPF}}
    &=\mathcal{L}_{\mathrm{pred}}
      +\lambda_{\mathrm{orth}}\mathcal{L}_{\mathrm{orth}}
      +\lambda_{\mathrm{fac}}\mathcal{L}_{\mathrm{fac}}
      +\lambda_{\mathrm{enc}}\mathcal{L}_{\mathrm{enc}},\\
    \mathcal{L}_{\mathrm{train}}^{(\delta)}
    &=\mathcal{L}_{\mathrm{base}}^{(\delta)}
      +\mathcal{L}_{\mathrm{OPF}}.
\end{aligned}
    \label{eq:objective}
\end{equation}
Each application combines its base objective with
$\mathcal{L}_{\mathrm{OPF}}$; the adapter, view sampler, encoder family, and
input tokenization specialize the predictive term to the data.

\subsection{Shared Predictive Pretraining and Task-Specific Readout}

Every instantiation follows the same optimization loop: (1) adapt a raw
observation into tokens and descriptors; (2) sample context and targets; (3)
encode context with the online encoder and targets with the stop-gradient EMA
encoder; (4) predict every orthogonal target factor; (5) update the online
encoder, projectors, and predictors using the base-plus-OPF training loss in
Equation~\eqref{eq:objective}; and
(6) update the target encoder by EMA. Thus, \method{} specifies one common
algorithm across domain pipelines.

\paragraph{Representation mode.}
Ordinary downstream evaluation retains the online encoder as a common source of
reusable features and attaches a readout that follows the structure of the
domain and task. The EMA target encoder, projectors, and prediction heads are
discarded. For domain $\delta$ and task $\tau$, we write
\begin{equation}
    h_{\delta,\tau}(x)
    =R_{\delta,\tau}\!\left(
    f_\theta^{\delta}(\mathcal{A}_{\delta}(x))
    \right),
    \label{eq:downstream}
\end{equation}
where $R_{\delta,\tau}$ may select a token, pool a set or sequence, preserve a
time-indexed state, or implement a learned probe. Orthogonal predictive
factorization therefore acts as a shared training principle that shapes a
reusable encoder while allowing each modality to use its own aggregation rule.

\paragraph{Operational world-model mode.}
For future-state prediction, intervention forecasting, planning, and
autoregressive simulation, the learned projectors and predictors are retained.
Equation~\eqref{eq:synthesis} then supplies the complete next latent state to a
decoder, planner, or subsequent transition step. Thus the same predictive
pretraining principle supports two uses: reusable encoder features for state
readout and an explicit latent transition interface for predictions composed
over time.

For factor-level analysis, we retain the learned projectors as a diagnostic
interface. Let $z_\delta(x)\in\R^d$ denote the online-encoder state obtained at
the same token or pooled level on which target factorization was trained. We
define
\[
    u_k(x)=P_k^\top z_\delta(x), \qquad k=1,\ldots,K.
\]
These factor coordinates expose distinct predictive modes for domain-specific
analysis, while ordinary downstream tasks use the encoder state directly. In the
scientific applications below, they provide the interface for intervention
analysis and spectral characterization.

\section{Experiments}
\label{sec:experiments}

We evaluate \method{} across heterogeneous visual, biological, clinical,
control, molecular, physical-field, and weather systems as three linked tests
of the proposed design. Each system defines its own observations and
context--target relation while retaining the same JEPA-Anything core. Group I
evaluates terminal readout: vision and single-cell tasks read the online encoder
state, while disease forecasting reads a one-step synthesized future state.
Group II tests whether predicted states can be repeatedly reused in
intervention-conditioned and autonomous dynamics. Group III tests whether factor
coordinates can be reused for scientific analysis and external validation.

\subsection{Common Evaluation Logic}
\label{sec:eval-logic}

The quantitative comparisons in Groups I and II combine domain-standard
baselines, monolithic JEPA baselines, and \method{}. In direct standard JEPA
comparisons, the adapter, encoder architecture, context--target sampler,
optimization budget, data split, downstream readout, and inherited base loss
are held fixed. The standard JEPA condition adds its monolithic predictive
term, whereas \method{} adds $\mathcal{L}_{\mathrm{OPF}}$, including the
factor-specific regularizers. Other domain baselines establish reference
points for the performance scale of each task. Table~\ref{tab:experiment_overview}
summarizes how the shared predictive interface is instantiated and evaluated.

\begin{table}[tp]
\centering
\caption{Overview of \method{} instantiations and evaluation protocols. In
Groups I and II, each implementation combines its base loss with
$\mathcal{L}_{\mathrm{OPF}}$ as specified in
Equation~\eqref{eq:objective}. Numerical OPF entries use the compact form
$d;\,K\!\times\!r$, where $Kr=d$; symbolic entries denote domain-configured
widths. Group III reports the factor interface used for analysis and external
validation.}
\label{tab:experiment_overview}
\fontsize{8.8}{10.4}\selectfont
\setlength{\tabcolsep}{2.4pt}
\renewcommand{\arraystretch}{1.02}
\begin{tabular}{@{}
>{\raggedright\arraybackslash}p{0.17\textwidth}
>{\raggedright\arraybackslash}p{0.25\textwidth}
>{\raggedright\arraybackslash}p{0.17\textwidth}
>{\raggedright\arraybackslash}p{0.12\textwidth}
>{\raggedright\arraybackslash}p{0.235\textwidth}@{}}
\toprule
\shortstack[l]{Group/domain\\and data} &
\shortstack[l]{Context $\rightarrow$ target\\or analysis} &
\shortstack[l]{Backbone or\\interface} &
\shortstack[l]{OPF setup /\\interface} &
\shortstack[l]{Seeds / splits /\\metrics} \\
\midrule
I / vision; controlled MuJoCo scenes & Visible patches $\rightarrow$ masked
patch states & \shortstack[l]{DINOv3 ViT-S\\SigLIP2 Base} &
\shortstack[l]{$384;\,4\!\times\!96$\\$768;\,4\!\times\!192$} & 10 seeds;
image-disjoint and leave-one-cell-out; INJ, Coll., Rec. \\
I / single cell; kidney, PBMC-10K, Adamson, Norman & Representation:
masked expression $\rightarrow$ complete cell state; Norman:
target $z_{\mathrm{response}}:=z_{\mathrm{observed}}-z_{\mathrm{control}}$
& scGPT & Norman: $512;\,4\!\times\!128$ &
5 seeds; kidney pretraining; PBMC fine-tuned/zero-shot and cross-dataset perturbation
transfer; AvgBIO, Pearson \\
I / clinical; longitudinal multimodal cohort & Patient history $\rightarrow$
future patient state & GPT-2 small; risk decoder &
$768;\,4\!\times\!192$ & 5 seeds; fixed patient-level cohort split; more than
1,000 event risks; PRAUC \\
\midrule
II / intervention; CITRIS Interventional Pong & Observation + intervention
$\rightarrow$ next state & Matched encoder and predictor &
$160;\,5\!\times\!32$ & 5 paired seeds; single- and combined-intervention
one-step prediction; six-step free rollout; four-channel MSE \\
II / dynamics; CausalWorld, DMC, PDEBench, WeatherBench2 & State, frames,
action, or forcing $\rightarrow$ next state or field & Matched task-specific
encoder and transition & $128;\,4\!\times\!32$ & 5 seeds;
held-out/OOD and 6-step rollout; MSE
and return \\
II / PDE audit; APEBench Burgers and KS & Current field $\rightarrow$ next
field & MLP; first-layer weight $256\!\times\!64$ &
$128;\,4\!\times\!32$ & 5 paired
seeds; held-out late-state prediction; Burgers 6-step rollout; MSE \\
II / long rollout; 1-D viscous Burgers, 64 points & Current field
$\rightarrow$ autoregressive future field & MLP; input $x_{\rm dim}=1024$;
hidden width 256 & $128;\,4\!\times\!32$ & 3 seeds;
2,000 training steps; ID/high-frequency OOD at H20/H50; MSE \\
II / locomotion; Hopper, Walker2d, HalfCheetah & State + action $\rightarrow$
future control state & Latent dynamics; CEM planner & $32;\,4\!\times\!8$ &
5 seeds; matched resets and 20-step factor interventions; return, $R^2$, MSE \\
II / molecules; water, quartz, paracetamol, benzene & Atomic state
$\rightarrow$ future molecular state & TrajCast-style $O(3)$ equivariant;
$l=1$ sector & $64;\,4\!\times\!16$ (channels) & 5 seeds; 100-step free
rollout; MAE, RMSD \\
\midrule
III / biology; Huh7--PBMC, organoids, tumor fragments, mice & Factor coordinates
$\rightarrow$ candidate intervention nomination & OPF factor-analysis
interface & Retained $P_k$ & 3 organoids and 3 fragments plus external mouse validation; killing and
activation \\
III / orbital; simulated position--velocity trajectories & Factor coordinates
$\rightarrow$ latent spectral modes & Trajectory encoding; spectral analysis &
Retained $P_k$ & One analyzed run with analytic-law validation; slope and $R^2$ \\
\bottomrule
\end{tabular}
\end{table}

The quantitative tasks use their domain metrics rather than averaging raw
values across incompatible scales. Continuous-control results are summarized
over five training seeds and matched environment resets. Group III instead
connects model-derived hypotheses or latent modes to biological experiments or
established physical laws.


\begin{figure}[t]
\centering
\includegraphics[width=\textwidth]{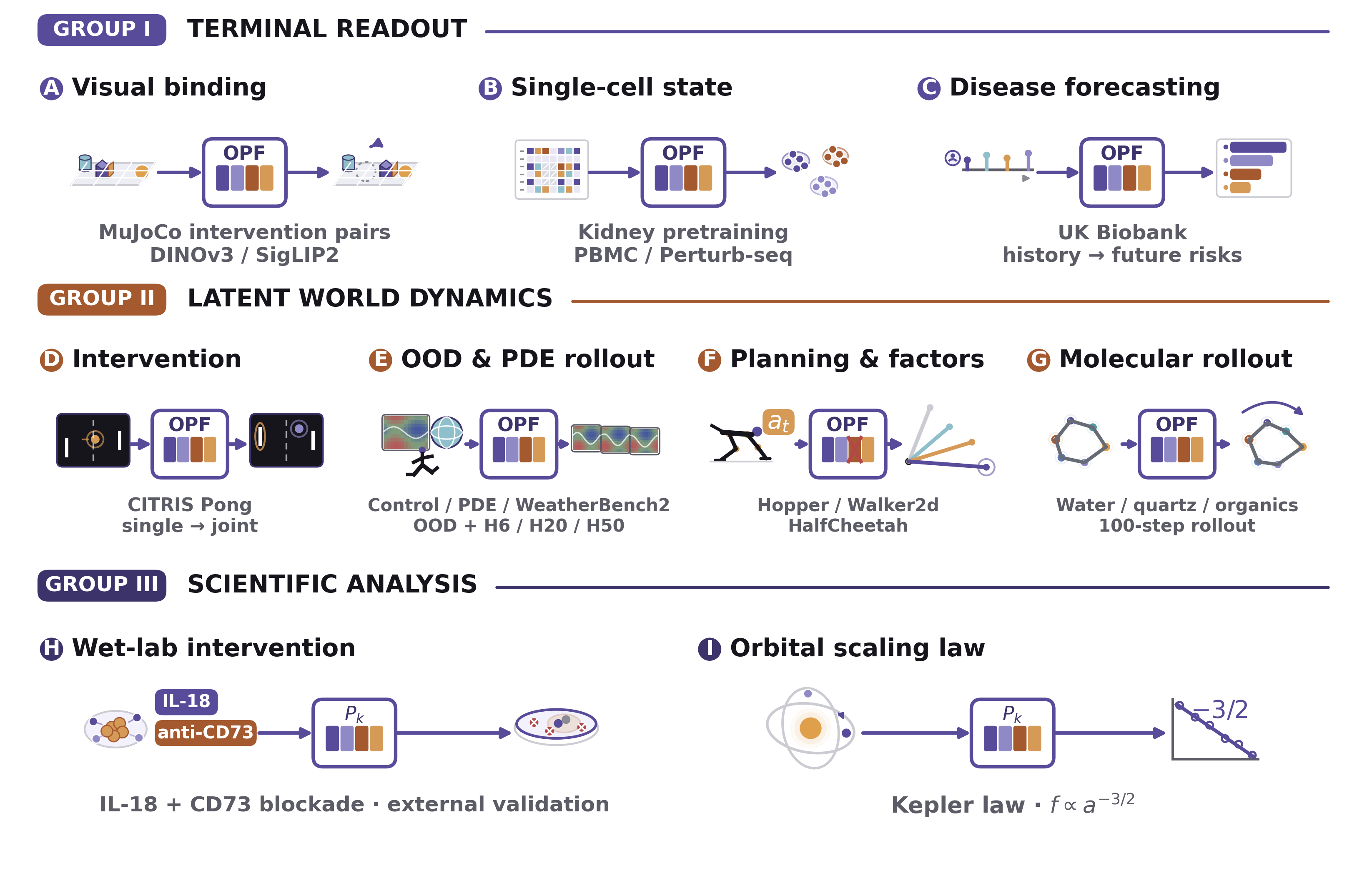}
\caption{Scenario atlas for \method{}. The layout follows the paper's three
evaluation groups: terminal readout, latent world dynamics, and scientific
analysis. Domain-specific contexts and outputs are connected through the same
factorized predictive-state interface; Group~III retains the learned factor
coordinates for external biological validation and physical-law analysis.}
\label{fig:scenario_atlas}
\end{figure}
\subsection{Group I: Terminal Readout of Predictive World States}

This group evaluates a predictive state at a terminal downstream readout. The
visual and single-cell experiments read the online encoder state directly. The
disease experiment predicts orthogonal factors, synthesizes one future patient
state, and reads event risks from that state. The distinction from Group II is
recursive use: Group I terminates at the readout, while Group II feeds predicted
states into subsequent transition or planning steps.

\subsubsection{Controlled Visual Binding}
\label{sec:cv}

\paragraph{Question.}
The first scenario tests whether orthogonal predictive factorization improves
the compositional readout of controlled visual changes. A frozen representation
must preserve both where a change occurs and which operation occurred, including
support--operation combinations withheld from readout training.

\paragraph{Predictive pretraining.}
For DINOv3 ViT-S ($d=384$) \cite{simeoni2025dinov3} and SigLIP2 Base
($d=768$) \cite{tschannen2025siglip2}, the frozen baseline directly uses the
original checkpoint. We use $K=4$, giving factor widths $r=96$ and $r=192$,
respectively. Standard JEPA and \method{} start from identical copies and are
trained by block-masked patch prediction with two-dimensional position
descriptors. They share the same data, masks, optimization settings, and
readout architecture; standard JEPA uses a monolithic target, whereas
\method{} uses the orthogonally factorized target. The resulting encoders are
frozen for controlled-intervention evaluation.

\paragraph{Readout and evaluation.}
For each source image and its controlled intervention, the patch-token
difference defines an innovation field. We evaluate reconstructed MuJoCo scenes
\cite{todorov2012mujoco} with image-disjoint splits, leave-one-cell-out readout, and
injective Hungarian alignment \cite{kuhn1955hungarian}. Known grid and learned
grid are alternative supervision conditions for the downstream structured grid
readout. Evaluating the frozen checkpoints under both separates supplied from
estimated grid structure. Standard JEPA and
\method{} use the same learned-grid readout, so their comparison isolates
encoder pretraining. We
report injective held-out-cell accuracy (INJ), collapse rate at threshold $0.5$
(Coll.), and grid recovery (Rec.), averaged over ten seeds. Recovery is not
applicable when the grid is supplied.

\begin{table}[t]
\centering
\caption{Controlled visual binding on MuJoCo, separating encoder pretraining
from downstream grid supervision. Known grid supplies the
support--operation taxonomy; learned grid estimates it from flat cell labels.
Standard JEPA and \method{} use the same learned-grid readout.}
\label{tab:cv_jepa_comparison}
\begin{tabular}{@{}lllrrr@{}}
\toprule
Backbone & \shortstack{Encoder\\pretraining} &
\shortstack{Grid\\supervision} & INJ $\uparrow$ &
Coll. $\downarrow$ & Rec. $\uparrow$ \\
\midrule
DINOv3 & Frozen checkpoint & Known grid & .452 & .667 & --- \\
DINOv3 & Frozen checkpoint & Learned grid & .569 & .433 & .643 \\
DINOv3 & Standard JEPA & Learned grid
& .572 & .426 & .645 \\
DINOv3 & \method{} & Learned grid
& .581 & .417 & .659 \\
\midrule
SigLIP2 & Frozen checkpoint & Known grid & .476 & .656 & --- \\
SigLIP2 & Frozen checkpoint & Learned grid & .484 & .511 & .676 \\
SigLIP2 & Standard JEPA & Learned grid
& .483 & .514 & .679 \\
SigLIP2 & \method{} & Learned grid
& .490 & .503 & .688 \\
\bottomrule
\end{tabular}
\end{table}

With the learned-grid readout held fixed, \method{} increases INJ and grid
recovery while reducing collapse for both DINOv3 and SigLIP2.

\subsubsection{Single-Cell State Representation and Perturbation Prediction}
\label{sec:cell}

\paragraph{Question.}
Single-cell RNA sequencing represents each cell through a high-dimensional but
extremely sparse expression profile. Gene-level reconstruction can therefore
devote capacity to technical dropout and count variability. This scenario asks
whether latent prediction learns a more stable representation of cellular
state, and whether orthogonal factorization further improves cell-type
clustering and perturbation-response prediction.

\paragraph{Protocol.}
Following Cell-JEPA \cite{elsheikh2026celljepa}, each cell is tokenized by gene
identity and discretized expression value using an scGPT backbone
\cite{cui2024scgpt}. The representation-learning path follows the masked
Cell-JEPA construction: a student encoder receives masked expression values and
predicts the cell-level embedding produced by an EMA teacher from the unmasked
cell. The Norman perturbation experiment predicts the latent response
$z_{\mathrm{response}}=z_{\mathrm{observed}}-z_{\mathrm{control}}$. Its
512-dimensional response representation uses $K=4$ factors of width $r=128$.
Cell-JEPA combines its monolithic latent loss with masked-gene reconstruction.
We re-evaluate it as the monolithic JEPA comparison in
Table~\ref{tab:cell_jepa}, while \method{} applies factorized latent prediction
to the target used by each experiment. Perturbation evaluation follows the
encoder $\rightarrow$ factor prediction heads $\rightarrow$ gating
$\rightarrow$ latent response residual $\rightarrow$ expression decoder path.

\paragraph{Evaluation.}
Models are pretrained on approximately 800,000 human kidney cells and evaluated
on PBMC-10K cell-type clustering under finetuned and zero-shot settings. We also
evaluate perturbation-response prediction on the Adamson
\cite{adamson2016perturbseq} and Norman \cite{norman2019perturbseq} datasets.
AvgBIO summarizes clustering quality, while Pearson correlation measures
absolute post-perturbation expression prediction. All reported results are
averaged over five training seeds.

\begin{table}[t]
\centering
\caption{Single-cell representation and perturbation evaluation. Perturbation
columns report Pearson correlation on absolute post-perturbation expression.}
\label{tab:cell_jepa}
\small
\setlength{\tabcolsep}{4pt}
\renewcommand{\arraystretch}{1.12}
\begin{tabular}{@{}lccrrrr@{}}
\toprule
Model & JEPA & \shortstack{Orthogonal\\factors} &
\shortstack{PBMC finetuned\\AvgBIO $\uparrow$} &
\shortstack{PBMC zero-shot\\AvgBIO $\uparrow$} &
\shortstack{Norman\\Pearson $\uparrow$} &
\shortstack{Adamson\\Pearson $\uparrow$} \\
\midrule
scGPT & No & No & 0.7531 & 0.5288 & 0.631 & 0.905 \\
Cell-JEPA & Yes & No & 0.7830 & 0.7194 & 0.787 & 0.937 \\
\method{} & Yes & Yes
& 0.8301 & 0.7752
& 0.814 & 0.942 \\
\bottomrule
\end{tabular}
\end{table}

\method{} attains the highest value on both PBMC clustering metrics and both
perturbation-response datasets.

\subsubsection{Longitudinal Disease-State Forecasting}
\label{sec:disease}

\paragraph{Question.}
Longitudinal health records are sparse, irregularly timed, and only partially
observe the physiological state that drives disease progression. This scenario
tests whether predicting future latent health states supports broad and
long-range clinical-event forecasting, and whether orthogonal factorization
improves prediction across the full event vocabulary.

\paragraph{Protocol.}
Each patient state combines time-ordered clinical events, demographic context,
and molecular or biochemical variables. The context encoder and its EMA target
copy use the GPT-2 small architecture with hidden width $d=768$
\cite{radford2019gpt2}; the former represents past history and the latter a
future patient state. We use $K=4$ factors of width $r=192$. A monolithic
health-state JEPA provides the direct comparison. The \method{} condition
predicts orthogonal factors and reconstructs the complete future state with
Equation~\eqref{eq:synthesis}. A shared lightweight decoder
maps the resulting state to risks for more than 1,000 future clinical events.
Figure~\ref{fig:disease_forecasting} compares the resulting model rankings.

\paragraph{Evaluation.}
Mean area under the precision--recall curve (PRAUC) across the full event
vocabulary is the primary metric because many outcomes are rare. All methods
use five training seeds and the same fixed patient-level cohort split.

\begin{figure}[t]
\centering
\includegraphics[width=0.92\textwidth]{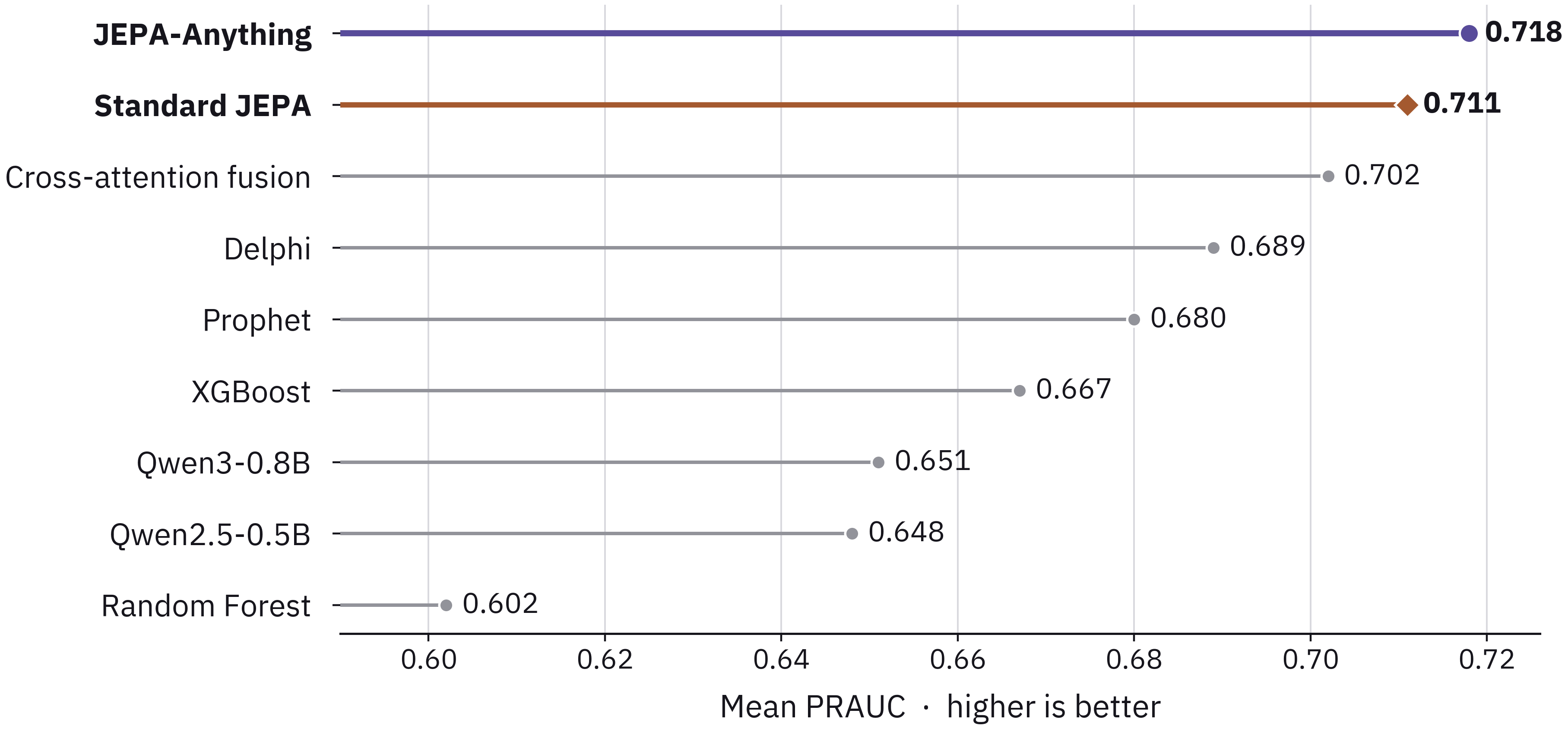}
\caption{Broad-spectrum prediction of more than 1,000 future clinical events.
Methods are ranked by mean PRAUC using the same cohort split and event
vocabulary; higher is better.}
\label{fig:disease_forecasting}
\end{figure}

The factorized future state yields a higher mean PRAUC than the matched
monolithic JEPA across the event vocabulary.

\FloatBarrier
\subsection{Group II: Latent World Dynamics and Rollout Stability}

This group tests the world-modeling role of the factorized state directly:
predicted latent states are used repeatedly by a planner or fed back through an
autoregressive rollout, so representation errors become dynamical errors. It
therefore evaluates the operational world-model mode. It includes controlled
and visual dynamics, physical fields, weather forecasting, closed-loop
planning, and molecular trajectories.

\subsubsection{Intervention-Conditioned Compositional Prediction}
\label{sec:citris-intervention}

\paragraph{Question.}
CITRIS Interventional Pong \cite{lippe2022citris} is a controlled dynamical
environment based on a simplified Pong game. Each sequence records how the
game state evolves and also identifies which underlying state factors were
externally changed between observations. Here, an intervention is not an action
chosen by the model: the experiment changes one or more factors of the system,
and the model receives both the pre-intervention observation and a label
specifying what was changed. Its task is to predict the resulting next state.
Figure~\ref{fig:citris_intervention} summarizes the paired prediction results.

We evaluate two levels of generalization. In the first, a single factor is
intervened on in a way represented in the training distribution. In the second,
multiple factors are intervened on together in a combination that was withheld
from training, even though the individual types of change were observed
separately. The latter condition therefore asks whether the model merely
memorizes complete intervention patterns or can recombine learned state-change
rules when they occur together in a new configuration. Lower prediction error
in both settings indicates more accurate intervention-conditioned dynamics and
better compositional reuse of previously learned changes. We additionally
evaluate six-step free rollout.

\paragraph{Protocol and evaluation.}
The dense standard JEPA baseline and \method{} use the same data, model budget,
and five fixed training seeds. We report the mean four-channel MSE for
single-intervention and combined-intervention one-step prediction, and for
six-step free rollout.

\begin{figure}[t]
\centering
\includegraphics[width=0.96\textwidth]{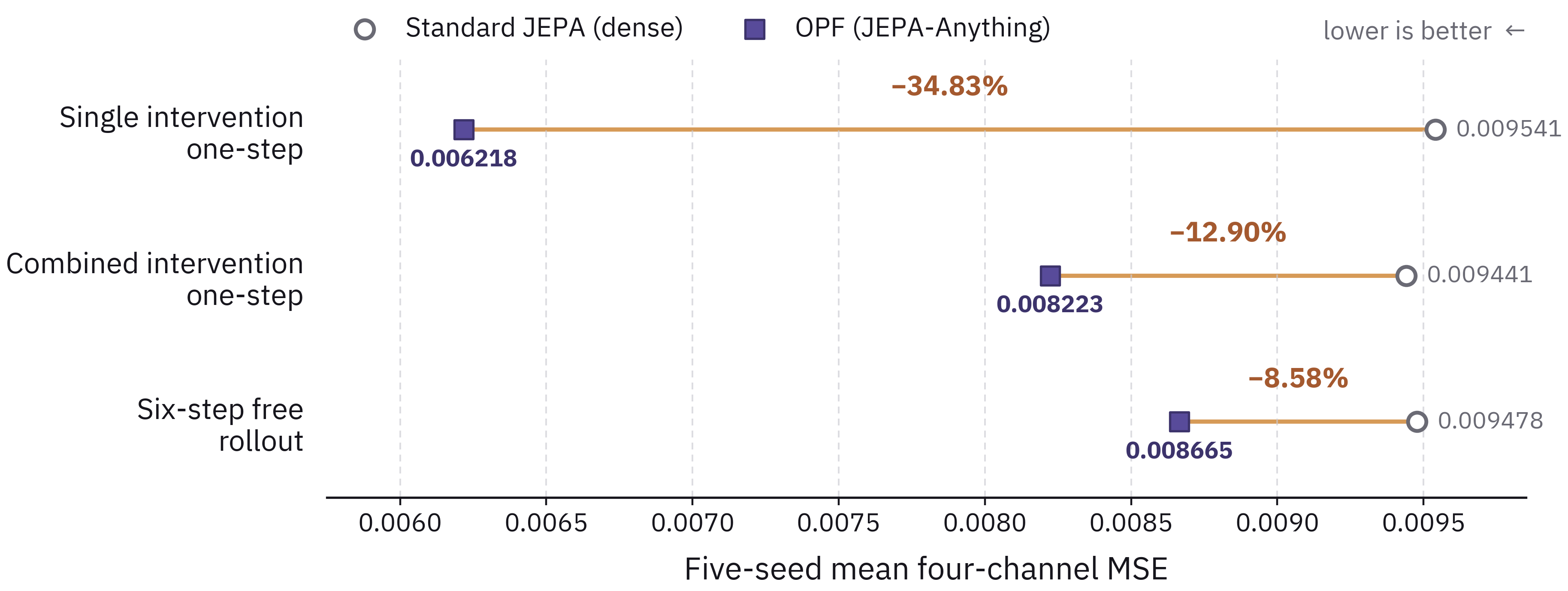}
\caption{Intervention-conditioned prediction and free rollout on CITRIS
Interventional Pong. Lower four-channel MSE is better. Points show means over
five paired seeds; annotations give the reduction computed from the reported
means.}
\label{fig:citris_intervention}
\end{figure}

Relative to the dense standard JEPA baseline, \method{} reduces
single-intervention one-step MSE from 0.009541 to 0.006218,
combined-intervention one-step MSE from 0.009441 to 0.008223, and six-step
free-rollout MSE from 0.009478 to 0.008665. The corresponding reductions are
34.83\%, 12.90\%, and 8.58\%, respectively.

\FloatBarrier
\paragraph{Orthogonality mechanism audit.}
We separately isolate the structural role of orthogonality by comparing the
factor interface with a capacity-matched unconstrained multi-head condition on
the same official CITRIS data, five training seeds, and 2,500-step budget. The
audit measures whether the learned factor blocks overlap, whether their
combined analysis matrix covers the state space with stable singular values,
whether the original latent state can be synthesized directly from the factor
coordinates, and whether this structure is obtained without deactivating most
factors.

\begin{table}[H]
\centering
\caption{Mechanism audit of the factor interface on CITRIS Interventional
Pong. Entries are mean $\pm$ sample standard deviation over five paired seeds.}
\label{tab:citris_geometry}
\footnotesize
\begin{tabular}{@{}lrr@{}}
\toprule
Metric & Unconstrained multi-head & Orthogonal factorization \\
\midrule
Cross-factor subspace overlap $\downarrow$
& $0.4550\pm0.0420$ & $(5.18\pm0.34)\!\times\!10^{-16}$ \\
Minimum singular value $\sigma_{\min}(P)$ $\uparrow$
& $0.00513\pm0.00233$ & $0.999989\pm0.000003$ \\
Condition number $\kappa_2(P)$ $\downarrow$
& $438.52\pm170.00$ & $1.00005\pm0.000004$ \\
\shortstack[l]{Transpose synthesis NMSE\\(exact factors) $\downarrow$}
& $0.7886\pm0.0299$ & $(2.98\pm0.11)\!\times\!10^{-14}$ \\
\bottomrule
\end{tabular}
\end{table}

The four metrics isolate complementary geometric properties. Cross-factor
overlap measures redundancy between factor subspaces; $\sigma_{\min}(P)$ measures
the weakest-covered state direction; and $\kappa_2(P)$ measures the sensitivity
of the coordinate system to errors. The reported synthesis NMSE uses exact
coordinates $u=P^\top z$ and transpose synthesis $Pu$, providing a measure of
factor-interface geometry independent of predictor error. Together, the results
show that orthogonal factorization
yields a self-dual, unit-conditioned coordinate system in which $P$ synthesizes
the state from $P^\top z$. Orthogonality therefore turns multiple prediction
outputs into a complete, non-overlapping, and numerically stable world-state
interface.

\subsubsection{Out-of-Distribution and Long-Horizon Dynamics}
\label{sec:ood-dynamics}

\paragraph{Question.}
This benchmark evaluates two capabilities required of a predictive world
model. First, after learning from a limited set of trajectories, the model must
predict states from unseen initial conditions, held-out episodes, or later time
periods. This out-of-distribution test asks whether the model has learned a
reusable transition rule rather than memorized the trajectories used for
training. Second, the model is rolled forward for six consecutive steps. At
each step its prediction becomes part of the input for the next prediction, so
small errors can compound over time. A model with more stable learned dynamics
should therefore maintain lower error as the rollout proceeds.

We apply the same evaluation principle to systems represented as moving pixels,
robot states and control-conditioned observations, and continuous physical or
weather fields. These settings are instantiated with CausalWorld
\cite{ahmed2021causalworld}, the DeepMind Control Suite
\cite{tassa2018dmcontrol}, PDEBench \cite{takamoto2022pdebench}, and
WeatherBench2 \cite{rasp2024weatherbench2}. The comparison tests whether
assigning complementary parts of a latent state transition to separate
predictive factors produces more accurate predictions on unseen states and
slows error growth during repeated prediction.

\paragraph{Protocol.}
The benchmark compares standard JEPA and \method{} over ten tasks and five
random seeds $\{11,23,37,53,71\}$. Standard JEPA uses one shared latent target,
whereas \method{} decomposes the target into four factors. Within each task and
seed, both variants use the same data, encoder and transition backbone,
optimization budget, evaluation split, and compute backend.

All rollout tasks receive the true contiguous six-step action or target
sequence, predict one step at a time, and compare each prediction with ground
truth at the corresponding time. WeatherBench2 normalization uses statistics
from the training period only, and control inputs use post-intervention
observations. Figure~\ref{fig:dynamics_summary} reports relative changes computed
from the five-seed means for the prediction tasks and mean $\pm$ sample standard
deviation for CausalWorld control.

\begin{figure}[t]
\centering
\includegraphics[width=\textwidth]{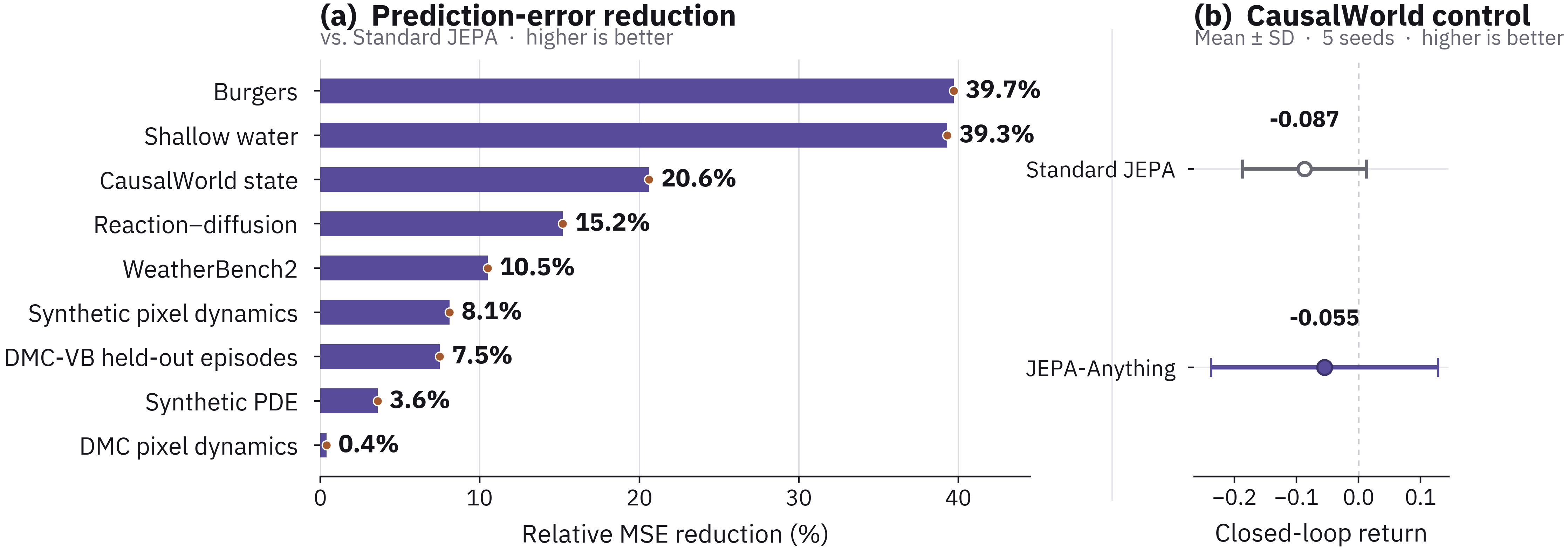}
\caption{Results from the ten-task matched dynamics benchmark. Panel (a)
shows the relative MSE reduction of \method{} versus standard JEPA, computed
from the reported five-seed means for the nine prediction tasks. Panel (b)
shows CausalWorld closed-loop control return as mean $\pm$ sample standard
deviation over five seeds; higher return is better.}
\label{fig:dynamics_summary}
\end{figure}

\FloatBarrier
\paragraph{Rollout and factor diagnostics.}
Table~\ref{tab:ood_rollout_positive} reports the first and sixth rollout steps
for four representative systems, making the accumulation of prediction error
explicit.

\begin{table}[H]
\centering
\caption{Six-step rollout diagnostics. Lower MSE is better. Each
entry is averaged over five seeds; the full evaluation records every
intermediate step.}
\label{tab:ood_rollout_positive}
\small
\begin{tabular}{@{}lrrrr@{}}
\toprule
& \multicolumn{2}{c}{Standard JEPA} & \multicolumn{2}{c}{\method{}} \\
\cmidrule(lr){2-3}\cmidrule(l){4-5}
Task & Step 1 & Step 6 & Step 1 & Step 6 \\
\midrule
CausalWorld state prediction & 0.01306 & 0.01307 & 0.01037 & 0.01079 \\
DMC pixel dynamics & 0.008539 & 0.008594 & 0.008498 & 0.008526 \\
PDEBench Burgers & 0.001830 & 0.006369 & 0.001101 & 0.004014 \\
PDEBench shallow water & 0.007090 & 0.010510 & 0.003999 & 0.006522 \\
\bottomrule
\end{tabular}
\end{table}

\paragraph{Additional APEBench evaluation.}
We additionally evaluate Burgers and Kuramoto--Sivashinsky (KS) dynamics under
the autoregressive PDE protocol of APEBench
\cite{koehler2024apebench}. We report these results separately because APEBench
and PDEBench use different normalizations. The APEBench models use an MLP with
a first linear-layer weight of shape $[256,64]$ and OPF dimensions
$d=128$, $K=4$, and $r=32$.
Table~\ref{tab:apebench_positive} reports held-out late-state prediction for
both systems and the six-step Burgers rollout.

\begin{table}[H]
\centering
\caption{Prediction results in the additional APEBench evaluation. Lower MSE
is better.}
\label{tab:apebench_positive}
\small
\begin{tabular}{@{}llrr@{}}
\toprule
System & Metric & Standard JEPA & \method{} \\
\midrule
Burgers & Held-out late MSE $\downarrow$
& $0.16621\pm0.01740$ & $0.08389\pm0.01770$ \\
Burgers & Six-step rollout MSE $\downarrow$
& $0.29186\pm0.00682$ & $0.16153\pm0.03500$ \\
Kuramoto--Sivashinsky & Held-out late MSE $\downarrow$
& $1.1012\pm0.0763$ & $0.9558\pm0.0340$ \\
\bottomrule
\end{tabular}
\end{table}

The Burgers errors decrease by approximately 49.5\% on the held-out late
segment and 44.7\% over six rollout steps, with improvement in every seed. KS
late-state prediction improves by approximately 13.2\%, also in every seed.
This extends the improvement to chaotic nonlinear dynamics.

\paragraph{Capacity-matched Burgers long-rollout evaluation.}
We additionally evaluate a parameter-matched 50-step protocol. Standard JEPA
and \method{} use the same 1-D
periodic viscous Burgers data with 64 grid points and $\nu=0.01$, and are
evaluated over three seeds on both in-distribution and high-frequency
out-of-distribution trajectories. The emulator receives an input of dimension
$x_{\rm dim}=1024$ and uses an MLP with hidden width 256. Its latent state has
$d=128$ dimensions and is divided into $K=4$ factors of width $r=32$, giving a
stacked projector shape of $(4,128,32)$. Models are trained for 2,000 steps.
Table~\ref{tab:burgers_long_eval} reports the 20- and 50-step errors.

\begin{table}[H]
\centering
\caption{Capacity-matched Burgers long-rollout evaluation over three seeds.
Lower error is better. Relative change is computed against standard JEPA.}
\label{tab:burgers_long_eval}
\small
\resizebox{\textwidth}{!}{%
\begin{tabular}{@{}lccc@{}}
\toprule
Split & Standard H20 / H50 & \method{} H20 / H50
& \method{} relative change \\
\midrule
ID & 1.1185 / 0.8181 & 1.0532 / 0.7923 & $-5.84\%$ / $-3.15\%$ \\
High-frequency OOD & 1.0649 / 0.8452 & 1.0116 / 0.8214
& $-5.01\%$ / $-2.82\%$ \\
\bottomrule
\end{tabular}}
\end{table}

The gain persists on both splits and at both horizons, including high-frequency
OOD trajectories, although it narrows from H20 to H50. The evaluation extends
the shorter APEBench rollout results to 50 prediction steps.

For the nine predictive tasks with routing statistics, \method{}
uses an average of $3.782$ effective factors out of four, with an average
dominant-factor share of $0.324$, indicating distributed activity across the
four branches.

\subsubsection{Continuous-Control Planning}
\label{sec:control}

\paragraph{Question.}
In model-based control, the agent uses the learned world model as an internal
simulator. Figure~\ref{fig:control_planning} reports the comparison between
standard JEPA and the full \method{} model. At each decision point, the CEM
planner generates many
candidate action sequences, rolls each sequence forward through the model,
estimates its future return, and executes the first action from the
highest-scoring sequence.
The model must therefore predict not only the immediate next state but a chain
of states under different possible actions. Even a small latent prediction
error can accumulate across this imagined trajectory, causing the planner to
rank a poor action sequence above a genuinely better one.

This experiment asks whether organizing the predicted state into orthogonal
factors improves the quality of these planning decisions. Standard JEPA and
\method{} are matched in parameter count and prediction cost, isolating the
effect of factorization. The final evaluation uses the return obtained in the
real environment because the central question is whether the learned latent
dynamics lead the planner to choose better actions.

\paragraph{Protocol.}
We evaluate Hopper-v5, Walker2d-v5, and HalfCheetah-v5
\cite{todorov2012mujoco}. A dense, capacity-matched standard JEPA and the full
\method{} model with $K=4$ use the same data, backbone, optimization budget,
and evaluation protocol. Training uses seeds $\{42,123,456,789,2026\}$. Each
checkpoint is evaluated on five identical environment resets, and paired
comparisons reset the CEM candidate-action randomness.

Branch widths are calibrated for $K\in\{1,2,4,8\}$. Relative to dense standard
JEPA, the training-parameter differences are $0.000\%$, $0.006\%$, $0.126\%$,
and $0.246\%$, while prediction-head FLOP differences are $0.000\%$, $0.742\%$,
$0.992\%$, and $1.449\%$, respectively, all below $0.3\%$ in parameters and
$2\%$ in FLOPs.

\begin{figure}[t]
\centering
\includegraphics[width=\textwidth]{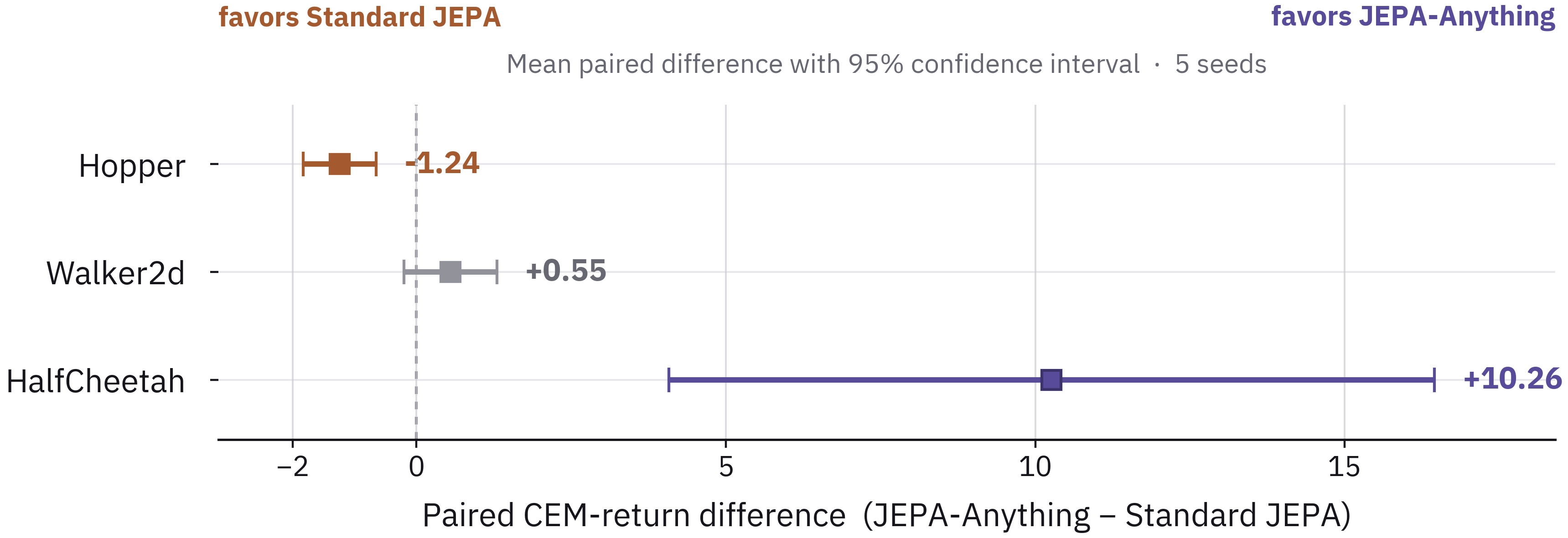}
\caption{Capacity-matched continuous-control planning. Squares show the mean
paired CEM-return difference between \method{} and standard JEPA over five
training seeds; error bars show 95\% confidence intervals. Positive values
favor \method{}.}
\label{fig:control_planning}
\end{figure}

\method{} attains higher mean CEM return on Walker2d and HalfCheetah, while
Hopper favors standard JEPA. Planning performance is therefore
environment-dependent. The factor-level interventions below test whether every
learned factor is functionally used during rollout and planning.

\FloatBarrier
\subsubsection{Factor Activity and Functional Interventions in Locomotion}
\label{sec:factor_validation}

\paragraph{Question and protocol.}
This experiment tests whether the four orthogonal groups in the locomotion
world model form active and individually usable state channels. For the full
$K=4$ model, the 32-dimensional state is partitioned into four learned
eight-dimensional subspaces. On held-out transitions, a separate ridge probe
uses one factor at a time to predict the next-step observation change
$\Delta o_t=o_{t+1}-o_t$. We additionally measure factor-wise standard
deviation, replace one factor by its training-set mean throughout a 20-step
rollout, and repeat the intervention during paired CEM evaluation in the real
environment.

\paragraph{Descriptive readout.}
The single-factor probes reveal differentiated motion preferences. On Hopper,
the most readable changes from F1--F4 are foot angle
($R^2=0.294$), torso height ($0.347$), torso angle ($0.424$), and leg angle
($0.310$), respectively. These are linearly accessible motion preferences
rather than preassigned one-to-one factor labels; one factor may retain
information about several related variables.

\paragraph{Activity and intervention.}
All four groups remain active. Across the analyzed checkpoints, their
factor-wise standard deviations lie in $0.692$--$0.715$ on Hopper,
$0.587$--$0.601$ on Walker2d, and $0.673$--$0.678$ on HalfCheetah. Figure
\ref{fig:factor_intervention} reports the change caused by masking one factor.
Positive heatmap values indicate increased 20-step MSE; negative CEM-return
changes indicate reduced real-environment return.

\begin{figure}[t]
\centering
\includegraphics[width=\textwidth]{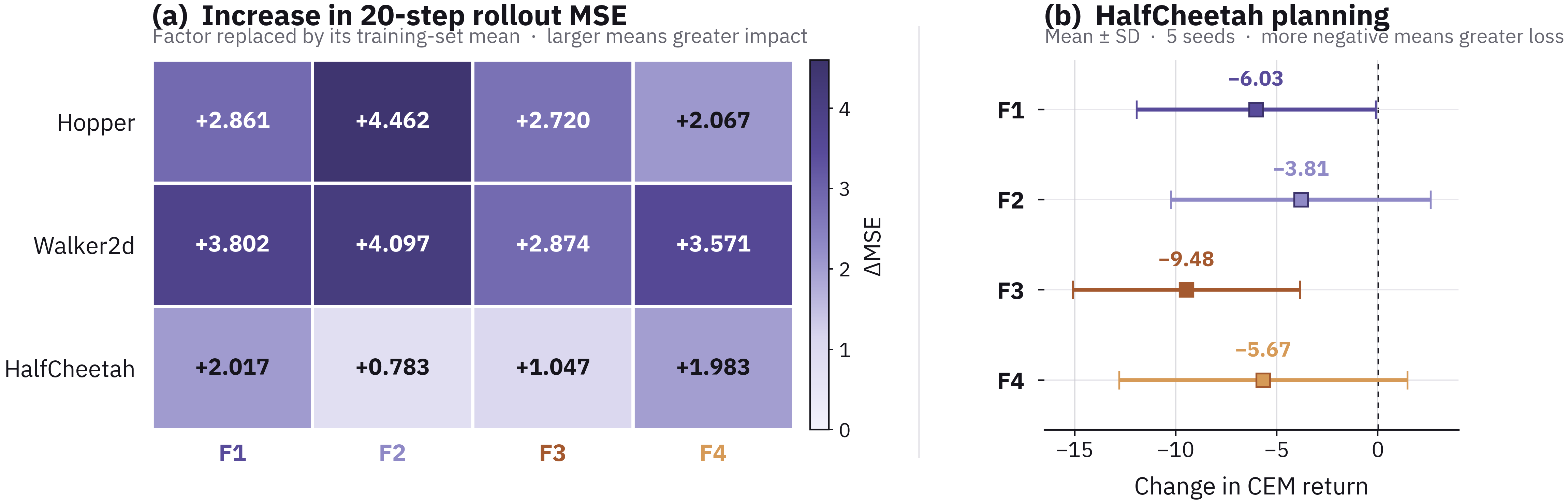}
\caption{Factor-wise functional interventions. Each factor is replaced by its
training-set mean throughout evaluation. Panel (a) gives the increase in
20-step rollout MSE. Panel (b) gives the HalfCheetah CEM-return change as mean
$\pm$ sample standard deviation over five training seeds.}
\label{fig:factor_intervention}
\end{figure}

Replacing a factor by its training-set mean removes sample-specific information
while keeping the intervened coordinate in a typical numerical range. Masking
any factor increases rollout error in all three environments. In HalfCheetah,
masking F1 or F3 reduces return in all five training seeds, while masking F2 or
F4 reduces return in four of five. Together, factor activity, single-factor
readout, and masking form an active--readable--functionally-used evidence chain.
They show that orthogonal predictive factorization supplies geometrically
non-overlapping, active, and individually intervenable state channels.

\FloatBarrier
\subsubsection{Force-Free Molecular Forecasting}
\label{sec:md}

\paragraph{Question.}
Force-free molecular forecasters directly predict future atomic positions and
velocities, then feed their own outputs back autoregressively. Small local
errors can therefore accumulate even when one-step prediction is accurate. This
scenario tests whether orthogonal future factors improve one-step accuracy and
long-horizon molecular forecasting stability.

\paragraph{Protocol.}
Each molecular state contains atomic positions, velocities, species, and the
simulation cell. We use a TrajCast-style equivariant forecaster
\cite{thiemann2026trajcast} on liquid water, crystalline $\alpha$-quartz,
gas-phase paracetamol, and benzene, and denote the JEPA-pretrained reference
condition as TrajCast-JEPA. \method{} retains the same system data,
$O(3)$-equivariant backbone, prediction intervals, supervised adaptation, and
rollout protocol while applying orthogonal predictive factorization during
pretraining. We apply OPF to the 64-channel multiplicity axis of the $l=1$
sector, using $d=64$, $K=4$, and $r=16$. Because an $l=1$ irrep has three
components, the corresponding flattened block contains $64\times3=192$
scalar coordinates; the learned channel projections act identically on the
three components, equivalently as $P_k\otimes I_3$, to preserve equivariance.

\paragraph{Evaluation.}
We report means over five independently trained models. For each seed, we compute the
one-step displacement MAE and the median final-position RMSD after 100 free
autoregressive steps, and then average each metric across the five seeds. The
latter is the primary measure of rollout stability; lower values are better.
Table~\ref{tab:molecular_forecasting} reports the values in \AA{}.

\begin{table}[t]
\centering
\caption{Force-free molecular forecasting with a TrajCast-style backbone.
Entries are means over five trained seeds. For each seed, we compute one-step
displacement MAE and median final-position RMSD after 100 autoregressive steps.
Values are reported in \AA{}; lower is better.}
\label{tab:molecular_forecasting}
\small
\begin{tabular}{@{}llrrr@{}}
\toprule
System & Metric $\downarrow$ & Scratch & TrajCast-JEPA & \method{} \\
\midrule
Water & MAE & 0.00387 & 0.00452 & \textbf{0.00376} \\
Water & RMSD & 3.331 & 2.536 & \textbf{2.459} \\
Quartz & MAE & 0.01080 & 0.01043 & \textbf{0.01011} \\
Quartz & RMSD & 2.089 & 1.912 & \textbf{1.877} \\
Paracetamol & MAE & 0.00901 & 0.00777 & \textbf{0.00705} \\
Paracetamol & RMSD & 3.155 & 1.868 & \textbf{1.776} \\
Benzene & MAE & $2.59\!\times\!10^{-5}$ & $2.10\!\times\!10^{-5}$
& $\mathbf{2.05\!\times\!10^{-5}}$ \\
Benzene & RMSD & 0.0958 & 0.0701 & \textbf{0.0645} \\
\bottomrule
\end{tabular}
\end{table}

\method{} produces the lowest one-step MAE and 100-step RMSD on all four
molecular systems.

\FloatBarrier
\subsection{Group III: Scientific Analysis of Learned World States}

This group tests factor coordinates as a scientific interface. The cancer case
uses them to nominate an intervention evaluated experimentally, while the
orbital case tests extracted modes against an established scaling law. Both
connect factor-level analysis of $u_k(x)$ to external validation.

\subsubsection{Wet-Lab Evaluation of a Factor-Nominated Intervention}
\label{sec:cancer-discovery}

Domain-specific intervention analysis applied to the orthogonal factor
coordinates learned by \method{} nominated IL-18 combined with
\emph{NT5E}/CD73 blockade as a candidate intervention. Wet-lab studies
subsequently supported this prediction, showing enhanced antitumor activity in
Huh7--PBMC co-culture, patient-derived organoids and tumor fragments, and
immunocompetent mice. This result illustrates how orthogonal factors learned by
\method{} can support the formulation of biologically testable hypotheses.

\begin{figure}[t]
\centering
\includegraphics[width=\textwidth]{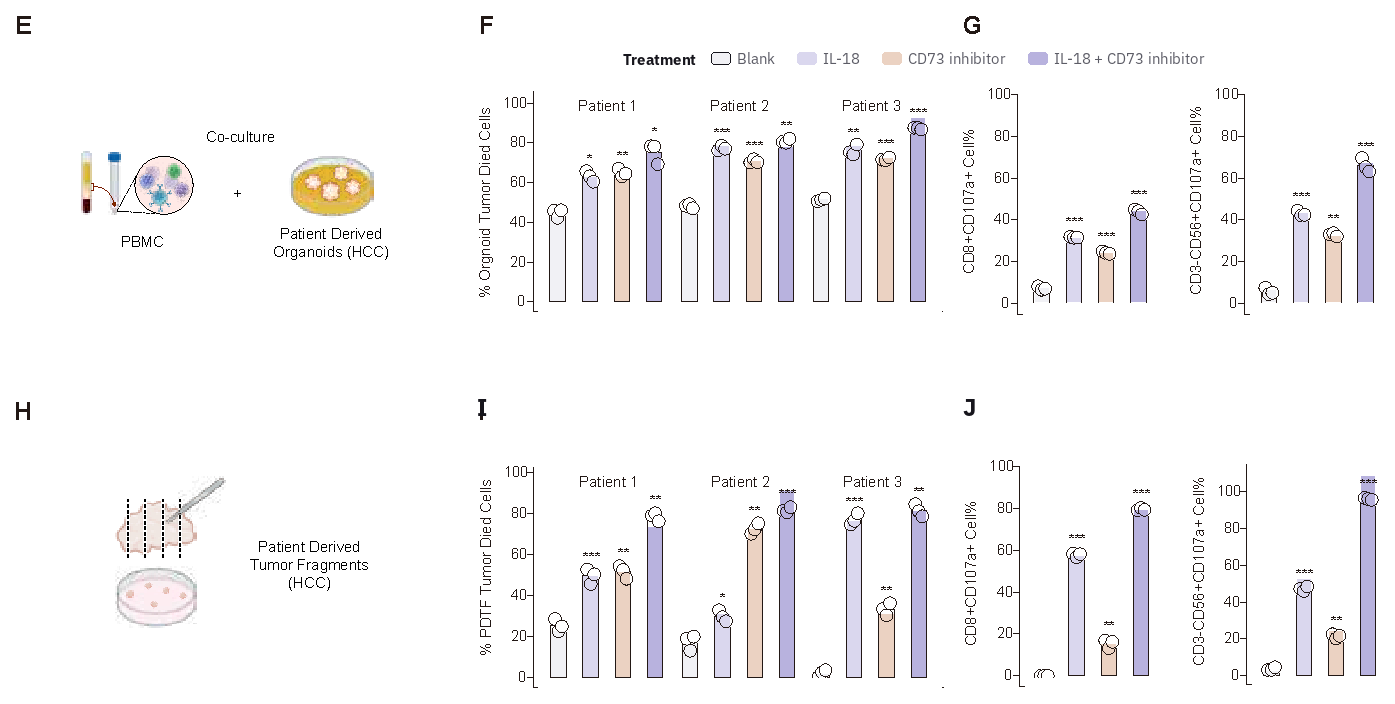}
\caption{Selected wet-lab results in patient-derived hepatocellular carcinoma
organoids and tumor fragments. Across three organoid specimens and three
tumor-fragment specimens, the IL-18 plus CD73 blockade combination showed the
strongest tumor-cell killing among the tested conditions, together with
increased immune-cell activation. Colors identify blank control, IL-18, CD73
inhibition, and the combined intervention.}
\label{fig:cancer_discovery}
\end{figure}

\FloatBarrier
\subsubsection{Orbital Scaling as a Physical-Law Diagnostic}
\label{sec:orbital}

Astronomical trajectories provide a complementary example in which predictive
latent modes are inspected for an interpretable law rather than used only for
forecasting. From simulated position--velocity trajectories without physical
labels, spectral analysis pairs latent frequencies with orbital semimajor axes.
In the \method{} analysis shown in Figure~\ref{fig:kepler}, the resulting modes
recover the Keplerian relation
$f\propto a^{-3/2}$ with high precision.

This example illustrates the scientific role of factor-level analysis: learned
predictive components can be evaluated against candidate symmetries,
invariants, or scaling relations. Here, agreement with the established
Keplerian law tests whether the orthogonal factors learned by \method{} capture
a physically meaningful dynamical regularity.

\begin{figure}[t]
\centering
\includegraphics[width=\textwidth]{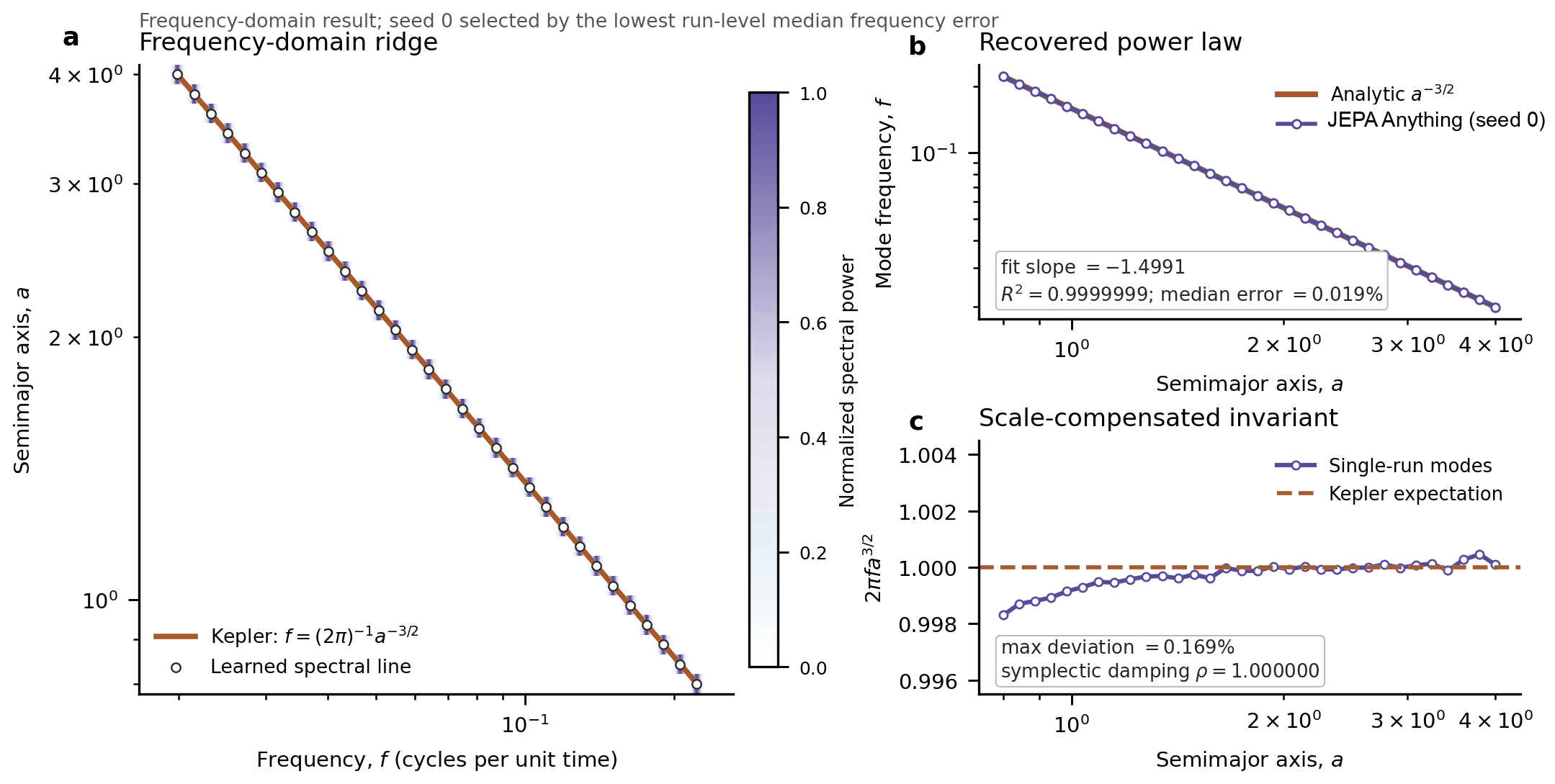}
\caption{A physical-law diagnostic based on predictive latent modes.
In a \method{} run, the learned spectral modes recover the Keplerian ridge
$f=(2\pi)^{-1}a^{-3/2}$ with fitted slope $-1.4991$ and $R^2=0.9999999$.}
\label{fig:kepler}
\end{figure}

\FloatBarrier
\section{Related Work}
\label{sec:related-work}

\paragraph{Latent world models and planning.}
Latent world models compress observations into states that can be advanced by
learned dynamics. Early recurrent world models coupled a latent dynamics model
to a controller \cite{ha2018worldmodels}, while DreamerV3 demonstrated that a
single model-based reinforcement-learning configuration can support behavior
learning across diverse domains \cite{hafner2025worldmodels}. Decoder-free
approaches such as TD-MPC2 optimize latent dynamics directly for scalable
continuous control \cite{hansen2024tdmpc2}. DINO-WM instead predicts features
from a pretrained visual encoder and plans toward visual goals from offline
trajectories \cite{zhou2025dinowm}. V-JEPA~2 combines large-scale
representation learning with an action-conditioned latent model for robotic
planning \cite{assran2025vjepa2}. These systems primarily develop latent states
for visual prediction or control.

\paragraph{Joint-embedding predictive learning.}
I-JEPA predicts representations of masked image regions from visible context
\cite{assran2023ijepa}, and V-JEPA extends feature prediction to video without
pixel reconstruction \cite{bardes2024vjepa}. Subsequent work scales this
principle to physical understanding and planning \cite{assran2025vjepa2}, while
Cell-JEPA adapts latent prediction to single-cell transcriptomics
\cite{elsheikh2026celljepa}. Analyses of JEPA objectives also show that their
learned invariances depend on which signals vary across context--target pairs
\cite{sobal2022slowfeatures}.

\paragraph{Structured and factorized representations.}
Factorized representation learning has commonly sought coordinates aligned
with independent generative variables, as in $\beta$-VAE
\cite{higgins2017betavae}, or entity-aligned slots, as in Slot Attention
\cite{locatello2020slotattention}. Structured world models similarly represent
states as interacting objects and relations \cite{kipf2020cswm}, while CITRIS
uses temporal interventions to identify causal factors
\cite{lippe2022citris}. Such semantic decompositions require corresponding
inductive biases or supervision; in general, unsupervised semantic
disentanglement is not identifiable from observations alone
\cite{locatello2019disentanglement}.

\paragraph{Redundancy reduction and collapse prevention.}
Barlow Twins reduces redundancy through cross-correlation matching
\cite{zbontar2021barlow}, and VICReg combines invariance with explicit variance
and covariance regularization \cite{bardes2022vicreg}. Across these lines, OPF
factorizes predictive target capacity without assuming object, causal, or
semantic factor identities. It applies orthogonality and activity constraints
to learned target-subspace bases and retains an explicit synthesis map, linking
redundancy reduction to a complete latent world state.

\section{Summary}
\label{sec:summary}

\method{} separates domain-specific predictive-state construction from a shared
OPF core. OPF assigns target capacity to complementary subspaces, keeps them
active, and synthesizes their predictions into a complete latent state. The
same latent world-state interface supports terminal encoder readout, recursive
transition and planning, or factor-level analysis.

Experiments evaluate these three uses across visual, biological, clinical,
control, molecular, physical-field, and weather systems. Terminal-readout tasks report stronger
visual, cellular, and disease-state metrics; CITRIS and the dynamics benchmarks
support intervention composition, OOD prediction, and repeated rollout. The
geometry analysis shows non-overlapping, unit-conditioned coordinates and
near-exact synthesis. Continuous-control means improve on two environments but
not Hopper, while factor interventions show that all locomotion channels
contribute to rollout. Biological and orbital cases further illustrate external
analysis of the learned coordinates. Together, the results support OPF as a
capacity-organizing core for a reusable latent world-state interface.

\section{Future Work}
\label{sec:future_work}

JEPA-Anything provides a common predictive learning principle across
different worlds. A broader direction is to turn these predictive models
into components of systems that actively investigate the worlds they
model. This connects naturally to Discovery Foundation Models
(DFMs)~\cite{yang2026discovery}, which frame discovery as a process of
formulating problems, constructing representations, proposing hypotheses,
and revising them through external evidence. Within this perspective,
JEPA-Anything could provide a predictive substrate for comparing
candidate interventions and identifying discrepancies between expected
and observed outcomes. Such discrepancies could guide new experiments
and, when necessary, motivate revisions to the variables, representations,
or hypotheses used to describe a system.

A concrete next step is to connect this predictive substrate with
ScienceIDE~\cite{geng2026scienceide} and
ScienceBuddy~\cite{xue2026sciencebuddy}.
ScienceIDE provides executable scientific environments in which agents
can run analyses, invoke simulations, and obtain verifiable feedback.
These environments could support the collection of action-conditioned
experience and the evaluation of decisions informed by learned world
models. ScienceBuddy provides an interactive research workspace and
couples harness evolution with model learning. Integrating predictive
models into this workflow would allow agents to translate researcher
questions into candidate experiments, use latent predictions to
prioritize them, and compare their outcomes with computational or
physical evidence. The resulting experience could inform subsequent
updates to the predictive model, agent policy, and research procedures.

Realizing this loop requires progress beyond predictive accuracy.
Key questions include how to calibrate uncertainty under distribution
shift, determine when a learned model is reliable enough to guide an
experiment, and revise predictive factors as new evidence reveals
missing structure. Orthogonal factors do not by themselves establish
causal mechanisms, and a shared predictive architecture does not
guarantee aligned representations across domains. Future evaluation
should therefore measure the quality and efficiency of experiment
selection, externally validated hypothesis revision, and transfer to
subsequent investigations. The long-term goal is a scientific learning
loop in which better world models enable more informative experiments,
and those experiments improve both the models and the agents that
use them.

\bibliography{references}

\clearpage
\appendix

\section{Approximate Pseudoinverse-Synthesis Stability}
\label{app:approximate-stability}

The exact orthogonal decomposition in
Proposition~\ref{prop:orthogonal-decomposition} has a continuous counterpart
when the learned analysis matrix is only approximately orthogonal and synthesis
uses the pseudoinverse in Equation~\eqref{eq:synthesis}.

\begin{proposition}[Approximate stability]
\label{prop:approximate-stability}
Suppose $P\in\R^{d\times d}$ satisfies
$\|P^\top P-I_d\|_2\leq\epsilon<1$. Then
\begin{equation}
 \sqrt{1-\epsilon}\leq\sigma_i(P)\leq\sqrt{1+\epsilon},
 \qquad
 \kappa_2(P)\leq\sqrt{\frac{1+\epsilon}{1-\epsilon}}.
 \label{eq:approx-condition}
\end{equation}
In particular, $P$ is invertible. For exact factor coordinates $u=P^\top z$
and predicted coordinates $\widehat u=u+e$, pseudoinverse synthesis obeys
\begin{equation}
 \widehat z=(P^\top)^\dagger\widehat u
 =z+(P^\top)^{-1}e,
 \qquad
 \|\widehat z-z\|_2
 \leq\frac{1}{\sqrt{1-\epsilon}}\|e\|_2.
 \label{eq:approx-synthesis}
\end{equation}
\end{proposition}

\begin{proof}
The eigenvalues of $P^\top P$ lie in
$[1-\epsilon,1+\epsilon]$, which gives
Equation~\eqref{eq:approx-condition} and implies
$\sigma_{\min}(P)\geq\sqrt{1-\epsilon}>0$. Hence $P$ is invertible and
$(P^\top)^\dagger=(P^\top)^{-1}$. Substituting
$\widehat u=P^\top z+e$ gives the identity in
Equation~\eqref{eq:approx-synthesis}; the bound follows from
$\|(P^\top)^{-1}\|_2=1/\sigma_{\min}(P)$.
\end{proof}

Thus pseudoinverse synthesis has no reconstruction bias for exact factor
coordinates, while the measured orthogonality residual controls amplification
of factor-prediction error.

\section{Supplementary Reproducibility Information}
\label{app:reproducibility}

\subsection{Implementation Details}

The APEBench Burgers and Kuramoto--Sivashinsky models use an MLP with a first
linear-layer weight of shape $[256,64]$ and OPF dimensions $d=128$, $K=4$, and
$r=32$. The Norman single-cell model uses $d=512$, $K=4$, and $r=128$, and
defines the prediction target as
$z_{\mathrm{response}}=z_{\mathrm{observed}}-z_{\mathrm{control}}$.

\subsection{OPF Coefficients}

Table~\ref{tab:app-opf-hyperparameters} reports the dimensions and three
additive OPF coefficients in Equation~\eqref{eq:objective} for the listed
experiments. Domain-specific terms remain part of
$\mathcal{L}_{\mathrm{base}}^{(\delta)}$. The dimensions agree with
Table~\ref{tab:experiment_overview}.

\begin{table}[H]
\centering
\caption{OPF dimensions and coefficients of the three additive regularizers in
Equation~\eqref{eq:objective}.}
\label{tab:app-opf-hyperparameters}
\scriptsize
\setlength{\tabcolsep}{2.5pt}
\renewcommand{\arraystretch}{0.96}
\begin{tabular}{@{}>{\raggedright\arraybackslash}p{0.43\linewidth}rrrrrr@{}}
\toprule
Experiment & $d$ & $K$ & $r$ & $\lambda_{\mathrm{orth}}$
& $\lambda_{\mathrm{fac}}$ & $\lambda_{\mathrm{enc}}$ \\
\midrule
CITRIS Interventional Pong & 160 & 5 & 32 & .10 & .05 & .02 \\
CausalWorld state/control & 128 & 4 & 32 & .10 & .05 & .02 \\
DMC pixel dynamics & 128 & 4 & 32 & .10 & .05 & .02 \\
DMC-VB held-out episodes & 128 & 4 & 32 & .10 & .05 & .02 \\
PDEBench Burgers & 128 & 4 & 32 & .10 & .05 & .02 \\
PDEBench reaction--diffusion & 128 & 4 & 32 & .10 & .05 & .02 \\
PDEBench shallow water & 128 & 4 & 32 & .10 & .05 & .02 \\
Synthetic PDE & 128 & 4 & 32 & .10 & .05 & .02 \\
Synthetic pixel dynamics & 128 & 4 & 32 & .10 & .05 & .02 \\
WeatherBench2 & 128 & 4 & 32 & .10 & .05 & .02 \\
APEBench Burgers / Kuramoto--Sivashinsky
& 128 & 4 & 32 & .10 & .05 & .02 \\
Hopper-v5, 20-step & 32 & 4 & 8 & .02 & .10 & .02 \\
Walker2d-v5, 20-step & 32 & 4 & 8 & .02 & .10 & .02 \\
HalfCheetah-v5, 20-step & 32 & 4 & 8 & .02 & .10 & .02 \\
\bottomrule
\end{tabular}
\end{table}

\subsection{AI Use Statement}

OpenAI Codex was used for language editing and translation; LaTeX
restructuring; consistency, bibliographic, and mathematical checks; cropping
and layout of manuscript figures; and feedback on experimental design,
statistical reporting, and the interpretation and presentation of results. The
authors reviewed all proposed changes and checked cited sources, derivations,
and experimental claims against the original materials. The authors made the
scientific decisions and take responsibility for the final manuscript.

\subsection{Ethics Statement}

The clinical experiment uses de-identified UK Biobank participant data accessed
through its controlled research and data-governance process. Analyses were
conducted within the scope of an approved research application and applicable
institutional ethics and data-use requirements. No participant-level data are
redistributed.

\subsection{Reproducibility Statement}

Section~\ref{sec:method} specifies the additive OPF loss, its combination with
the original implementation loss, the activity statistics, and the synthesis
map. Section~\ref{sec:experiments} identifies the data,
matched-baseline controls, and metrics, while
Table~\ref{tab:experiment_overview} records context--target construction,
factor dimensions, evaluation splits, and seed counts. Appendix
Table~\ref{tab:app-opf-hyperparameters} reports the OPF regularization
coefficients for the listed experiments. Experiments use the default
architecture and optimization settings of the cited backbone implementations,
with non-OPF settings shared within each matched comparison. Tables report
absolute metrics where scales are comparable, and figures state the aggregation
used for cross-task summaries.

\newpage
\section{Author Organizations and Correspondence}
\label{app:author-information}

\subsection*{Organizations}

\begin{description}[style=multiline,labelwidth=1.5em,leftmargin=2.2em,
  itemsep=0.35em,topsep=0.25em]
\item[\textsuperscript{1}] PhAI Labs.
\item[\textsuperscript{2}] The Chinese University of Hong Kong.
\item[\textsuperscript{3}] Fudan University.
\item[\textsuperscript{4}] City University of Hong Kong.
\item[\textsuperscript{5}] University of Bristol.
\item[\textsuperscript{6}] Stanford University.
\item[\textsuperscript{7}] University of Oxford.
\item[\textsuperscript{8}] Princeton University.
\end{description}

\subsection*{Correspondence}

\noindent *Correspondence: 
\texttt{pheng@cse.cuhk.edu.hk} (P.A.H.);
\texttt{yin@phai-labs.com} (Z.Y.);
\texttt{wuyc@phai-labs.com} (Y.W.); and
\texttt{yang@phai-labs.com} (L.Y.).

\end{document}